\documentclass[twoside]{article}

\usepackage[accepted]{aistats2023}
\usepackage{amsmath}
\usepackage{amsfonts}
\usepackage{amsthm}
\usepackage{graphicx}
\usepackage{tabularx}
\usepackage{xcolor}
\usepackage{mathtools}
\usepackage{subfig}
\definecolor{salmon}{rgb}{0.918,0.420,0.4}
\definecolor{teal}{rgb}{0.404, 0.671, 0.624}
\definecolor{teal1}{rgb}{0.15, 0.15, 0.424}
\usepackage[round]{natbib}
\usepackage{hyperref}
\hypersetup{
colorlinks,
citecolor=blue
,
urlcolor=blue,
bookmarks=false,
hypertexnames=true
}

\begin{document}
\twocolumn[

%\aistatstitle{Instructions for Paper Submissions to AISTATS 2023}
\aistatstitle{Neural Discovery  of Permutation Subgroups}

\aistatsauthor{ Pavan Karjol \hspace{2.5em} Rohan Kashyap \hspace{2.5em} Prathosh A P}

\aistatsaddress{ Department of Electrical Communication Engineering,\\ Indian Institute of Science, Bengaluru, Karnataka } ]

\begin{abstract}
We consider the problem of discovering subgroup $H$ of permutation group $S_n$. Unlike the traditional $H$-invariant networks wherein $H$ is assumed to be known, we present a method to discover the underlying subgroup, given that it satisfies certain conditions. Our results show that one could discover any subgroup of type $S_k (k \leq n)$ by learning an $S_n$-invariant function and a linear transformation. We also prove similar results for cyclic and dihedral subgroups. Finally, we provide a general theorem that can be extended to discover other subgroups of $S_n$. We also demonstrate the applicability of our results through numerical experiments on image-digit sum and symmetric polynomial regression tasks.
\end{abstract}

\section{INTRODUCTION}
\subsection{Background}
Deep Learning has proven to be a successful paradigm for learning the underlying regularities of sensory data such as images, text, and audio \citep{brown2020language, he2016deep, ramesh2022hierarchical}. The data in the physical world possess a predefined structure with a low-dimensional manifold approximation within a higher dimensional euclidean space \citep{cayton2005algorithms, scholkopf1998nonlinear}. However, the task of supervised learning in such a high-dimensional data space demands a large number of data points to counter the curse of dimensionality. Thus, universal function approximations using neural networks in such a setting can be prohibitively expensive to curate large datasets for diverse applications such as medical imaging. This calls for the need for inductive bias to be incorporated into our networks such that they can utilize these priors for learning valuable representations in the feature space.
Convolutional Neural Networks proposed by \citep{lecun1995convolutional} incorporate translation equivariance and thus preserve translation symmetry. This is highly effective for perception tasks since it enables the model with a notion of locality and symmetry, i.e., the input and label are both invariant to shifts (preserves this property across layers), and has likewise shown substantial gains in image recognition tasks as demonstrated in \citep{szegedy2017inception, he2016deep}. However, from a group-theoretic perspective, CNN happens to represent a particular case of invariance under the action of a specific group. This leads to studying and understanding its usability when extended to a more general setting, i.e., equivariance or invariance to any generic group action. Thus, learning such representations across neural nets ensures preserving symmetry across the network and efficiently discovering the underlying factors of data variations by utilizing these priors.

\subsection{Group Invariance and Equivariance}
Learning symmetries from data has been studied extensively in \citep{senior2020improved, raviv2007symmetries, monti2017geometric, rossi2022learning}. Invariant and equivariant classes of functions impose a powerful inductive prior to our models in a statistically efficient manner which aids in learning useful representations on a wide range of data \citep{bogatskiy2020lorentz, esteves2020theoretical}. 
Group equivariant or invariant networks \citep{cohen2018spherical, esteves2018learning} exploit the inherent symmetrical structure in the data, i.e., equivariance or invariance to a certain set of group operations (geometric priors) and can thus result in a significant reduction in the sample complexity and lead to better generalization. This has ubiquitous applications in various tasks such as predicting protein interactions \citep{gainza2020deciphering} and estimating population statistics \citep{zaheer2017deep}.

One of the important classes of group invariance networks corresponds to the permutation group $(S_n)$, i.e., the group of all permutations of a set of cardinality $n$. \cite{zaheer2017deep} have focused extensively on the applicability of permutation equivariance and invariance functions on arbitrary objects such as sets. Whereas, \citep{kicki2020computationally} proposes a $G$-invariant network to approximate functions that are invariant under the action of any given permutation subgroup of  $S_n$. Moreover, it is crucial to consider subgroups of $S_n$, since any finite group is isomorphic to a subgroup of $S_n$ \textit{(Cayley's theorem)} for some $n$. For example, the Quarternanian group $Q_8$ is isomorphic to a subgroup of $S_8$. In addition, other interesting applications of functions correspond to subgroups of $S_n$. For instance, the area of an $n$-polygon is a $\mathbb{Z}_n$-invariant function of the polygon's vertices \citep{kicki2020computationally}.

\subsection{Contributions}
In most of the works mentioned earlier, the group (or subgroup) is assumed to be known a priori. This restricted form of modeling choice leads to reduced flexibility (also restrictions). It makes incorporating symmetries into our networks highly infeasible for real-world applications where the underlying structure is unknown. Motivated by this, we demonstrate a general framework, i.e., $G$-invariant network and a linear transformation for discovering the underlying subgroup of $S_n$ under certain conditions. Our main contributions can be summarized as follows: 

In this work, we propose a general framework to discover the underlying subgroup of $S_n$ under a broad set of conditions.

\begin{itemize}
    \item We prove that we could learn any conjugate group (with respect to $G$) via a linear transformation and $G$-invariant network.
    \item We extend this approach, i.e., a linear transformation and $G$-invariant network to different classes of subgroups such as permutation group of $k$ (out of $n$) elements $S_k$, cyclic subgroups $\mathbb{Z}_k$ and dihedral subgroups $D_{2k}$. The  $G$-invariant networks for the above families are $S_n, \mathbb{Z}_n$ and $D_{2n}$ respectively. In the latter two cases, $k$ should divide $n$.
    \item We prove a general theorem that can guide us to discover other classes of subgroups.
    \item We substantiate the above results through experiments on image-digit sum and symmetric polynomial regression tasks.
\end{itemize}

\section{PRIOR WORK}
\subsection{Group Invariant and Equivariant Networks}
Significant progress has been made in incorporating invariances to deep neural nets in the last decade \citep{cohen2019general, cohen2016steerable, ravanbakhsh2017equivariance, ravanbakhsh2020universal, wang2020incorporating}. We observe that most of the invariant neural networks proposed in the literature assume the knowledge of the underlying symmetry group. Various generalizations, i.e., group equivariant or invariant neural networks, are presented in \citep{cohen2019general, kondor2018clebsch}. 

\cite{cohen2016group} introduce \textit{Group Equivariant Convolutional Neural Networks (G-CNNs)} as a natural extension of the Convolutional Neural Network to construct a representation with the structure of a linear G-space. Further, \cite{cohen2019general} presents a general theory for studying G-CNNs on homogeneous spaces and illustrates a one-to-one correspondence between linear equivariant maps of feature spaces and convolutions kernels. \cite{cohen2016steerable} provides a theoretical framework to study steerable representations in convolutional neural networks and establish mathematical connections between representation learning and representation theory.  \cite{ravanbakhsh2020universal} presents the universality of invariant and equivariant MLPs with a single hidden layer. Additionally, they show the unconditional universality result for Abelian groups. \cite{kondor2018generalization} utilize both representation theory and noncommutative harmonic analysis to establish the convolution formulae in a more general setting, i.e., invariance under the action of any compact group.

\subsection{Permutation Invariant and Equivariant Networks}
\cite{zaheer2017deep} demonstrates the applicability of equivariant and invariant networks on various set-like objects. Further, they show that any permutation invariant function can be expressed in a standard form, i.e., $\rho \left( \sum_i \phi \left (x_i \right) \right)$, which corresponds to an elegant deep neural network architecture. Janossy pooling \citep{murphy2018janossy} extends the same to build permutation invariant functions using a generic class of functions. The works, as mentioned earlier, focus mainly on the permutation group $S_n$. 

Recent works by \cite{kicki2020computationally} and \cite{maron2019universality} provide a general architecture invariant to any given subgroup of $S_n$. \cite{kicki2020computationally} design a $G$-invariant neural network for approximating functions (can specifically approximate any \textit{G-invariant function}) $f : X \rightarrow R$ using $G$-equivariant network and sum-product formulation, where $X$ is a compact subset of $R^{n \times m}$, for some $n$, $m > 0$) for any given permutation subgroup $G$ of $S_n$. They extend this work to study the invariance properties of hierarchical groups $G < H \leq S_n$. However, in most cases, the underlying subgroup is generally unknown.

\subsection{Automatic Symmetry Discovery}
\cite{dehmamy2021automatic} introduces the \textit{Lie algebra convolutional network (L-Conv)}, an infinitesimal version of G-Conv, for automatic symmetric discovery. Their framework for continuous symmetries relies on \textit{Lie algebras} rather than \textit{Lie groups} and can thus encode an infinite group without discretizing \citep{cohen2016group} or summing over irreps. They show that the $L$-Conv network can serve as a building block for constructing any group equivariant feedforward architecture. They also unveil interesting connections between equivariant loss and Lagrangians in field theory and robustness and Euler-Lagrange equations. However, these apply only to \textit{Lie groups} and are not specific to subgroups of the permutation groups. \cite{anselmi2019symmetry} proposes to learn symmetry-adapted representations and also deduce a regularization scheme for learning these representations without assuming the knowledge of the underlying subgroup (of $S_n$). However, their proposed solution is implemented in an unsupervised way. \cite{benton2020learning} and \cite{zhou2020meta} also propose different methods for learning symmetries when the group $G$ is unknown.

\section{PRELIMINARIES}
This section gives a brief overview of various mathematical concepts used in our work. Let $G$ be a group.
\begin{enumerate}
    \item \textbf{Group action}
    :- The action of $G$ on a set $X$ is defined using the following map (written as $g \cdot x, \; \forall g \in G \text{ and } x \in X$) :
    \begin{equation}
        \theta: G \times X \rightarrow X,
    \end{equation}
    satisfying the following properties :
    \begin{itemize}
        \item $g_1 \cdot (g_2 \cdot x) = (g_1 g_2) \cdot x  \quad \forall g_1, g_2 \in G \text{ and } x \in X$,
        \item $1 \cdot x = x, \quad \forall x \in X$
    \end{itemize}
    where $1$ is the identity element of $G$.
    %In summary, it defines various bijections of the set $X$.
    \item \textbf{Group invariant function} 
    :- A function  $f : X \rightarrow Y$ is said to be group invariant with respect to $G$, if,
    \begin{equation}
        f(x) = f(g \cdot x), \quad \forall g \in G \text{ and }  x \in X
    \end{equation}
    We call $f$ a $G$-invariant function.
    \item \textbf{Group equivariant function} 
    :- A function  $f : X \rightarrow Y$ is said to be group equivariant with respect to $G$, if for any $g \in G$,  $\exists$ $\Tilde{g} \in G$, such that
    \begin{equation}
        f(g \cdot x) = \Tilde{g} \cdot f( x),   \forall x \in X
    \end{equation}
    We call $f$ a $G$-equivariant function.    
    \item \textbf{Conjugate subgroups}
    :- Two subgroups $G_1$ and $G_2$ of  $G$ are said to be conjugates, if $\exists g \in G$ such that,
    \begin{equation}
        G_2 = g G_1 g^{-1} := \{ gkg^{-1} : k \in G_1 \}
    \end{equation}
    \item \textbf{Normal subgroup}
    :- A subgroup $N$ is said to be normal in $G$, if $\forall  g \in G$ 
    \begin{equation}
         g N g^{-1} = N
    \end{equation}
    i.e., there are no subgroups that are conjugate to $N$.
\end{enumerate}

We describe the notations used for various subgroups of $S_n$ in Table (\ref{notations}). Henceforth, unless explicitly mentioned, we follow the notations mentioned in Table (\ref{notations}). 

\begin{table}[htp]
    \centering
    \caption{Descriptions of notations}
    \begin{tabularx}{\columnwidth} { 
      | >{\raggedright}X 
       >{\raggedright\arraybackslash}X| }
     \hline
     \textbf{Symbol}  &\hspace{-2.5cm}\textbf{Description} \\
     \hline
    $S_n$                       &\hspace{-2.8cm}Permutation group of $n$ elements\\
    $S^{(0)}_k$                 &\hspace{-2.8cm}Permutation subgroup of first $k$ elements \\
    %&\hspace{-2.8cm}(keeping the remaining $n-k$ elements fixed)\\
    $S_k$                       &\hspace{-2.8cm}Permutation subgroup of random $k$ elements \\
    $\mathbb{Z}_n$              &\hspace{-2.8cm}Cyclic subgroup of $n$ elements\\
    $\mathbb{Z}^{(0)}_k$        &\hspace{-2.8cm}Cyclic subgroup of first $k$ elements \\
    $\mathbb{Z}_k$              &\hspace{-2.8cm}Cyclic subgroup  of random $k$ elements \\
    $D_{2n}$                    &\hspace{-2.8cm}Dihedral subgroup of $n$ elements\\
    $D^{(0)}_{2k}$              &\hspace{-2.8cm}Dihedral subgroup of first $k$ elements \\
    $D_{k}$                     &\hspace{-2.8cm}Dihedral subgroup of random $k$ elements \\
    $A_{n}$                     &\hspace{-2.8cm}Alternating subgroup of  $n$ elements \\
    $A_{k}$                     &\hspace{-2.8cm}Alternating subgroup of random $k$ elements \\
    \hline
    \end{tabularx}
    \label{notations}
\end{table}

\section{PROPOSED WORK}
\subsection{Problem statement}
We consider the problem of learning an $H$-invariant function $f : X \rightarrow \mathbb{R}$, where $X=[0,1]^n \subset R^n$ and $H$ is the unknown subgroup of $S_n$. In general, learning such a function is intractable. However, we show that it is possible to learn such a function, i.e., discover the underlying subgroup $H$, where $H$ belongs to a certain class of subgroups (we explicitly state our conditions in Theorem \ref{permutation groups}, \ref{Cyclic groups and Dihedral groups} and \ref{generalization}). The general consequence of our analysis is that learning a $H$-invariant function is thus equivalent to learning a $G$-invariant function along with a linear transformation, given that $G$ and $H$ satisfy certain conditions. Since any given $G$ can have several such subgroups, we propose to learn the underlying subgroup $H$ by exploiting the existing structures using a family of $G$-invariant functions (such as the one mentioned in \cite{zaheer2017deep} for the permutation group $S_n$) and a learnable linear transformation. We formalize these ideas in the coming subsections.

To prove our results, we employ the following theorem regarding $S_n$-invariant functions  \citep{zaheer2017deep}, which shows that any such function can be expressed in a canonical form.

 \newtheorem{theorem}{Theorem}[section]
\newtheorem{corollary}{Corollary}[theorem]
\newtheorem{lemma}[theorem]{Lemma}

\begin{theorem}[Deep sets]
$f: X = [0,1]^n \rightarrow \mathbb{R}$ is a permutation invariant ($S_n$-invariant) continuous function iff if has the representation,
\begin{equation}
    f(x) = \rho \left(\sum_{i=1}^n \gamma(x_i) \right), \; x=[x_1, x_2, \dots x_n]^T
\end{equation}
for some continuous outer and inner functions $\rho : \mathbb{R}^{n+1} \rightarrow \mathbb{R}$, $\gamma: [0,1] \rightarrow \mathbb{R}^{n+1}$.
\label{deepsets result}
\end{theorem}

%If we consider the permutations with respect to the first $k$ elements, we get the following result.
We get the following result if we consider the permutations of the first $k$ elements.
\begin{corollary}
$f: [0,1]^n \rightarrow \mathbb{R}$ be an $S_k^0$-invariant continuous function  iff it has the representation,
\begin{equation}
    f(x) = \rho \left( \sum_{i=1}^k \gamma(x_i), \; x_{k+1}, \dots, x_n  \right)
\end{equation}
\label{Corr Sk}
\end{corollary}
\begin{proof}
To prove Theorem \ref{deepsets result}, it has been shown that \citep{zaheer2017deep}, $\mathcal{X}^{(n)} = \{x_1, x_2, \dots, x_n \subset [0,1]^n : x_1 \leq x_2 \leq x_3 \dots \leq x_n \}$ is homeomorphic to $\sum_{i=1}^n \gamma(X_i)$, where 
\begin{equation}
 \gamma(t) = \big[1, t, t^2, \dots t^n \big]^T  
 \label{gamma definition}
\end{equation}
Hence,  $\mathcal{X}^{(n:k)} = \{x_1, x_2, \dots, x_n \subset [0,1]^n : x_1 \leq x_2 \leq x_3 \dots \leq x_k \}$  is homeomorphic to $\sum_{i=1}^k \gamma(X_i) \times [0,1]^{n-k}$.
Let, $E(x) = \Big[\sum_{i=1}^k \gamma(x_i), \; x_{k+1}, \dots, x_n  \Big]^T$. Then, it is an homeomorphism  from $\mathcal{X}^{(n:k)}$ to $Im(E)$ (Image of E). If we set $\rho = fE^{-1}$, we get $\rho\left(E(x)\right) = f(x)$.
\end{proof}
We use the same definition of $\gamma$ \citep{zaheer2017deep} provided in the eq. (\ref{gamma definition}) in the subsequent results as well.
Now, we state our first result using the conjugacy relation between subgroups.
\begin{lemma}
\label{Lemma 4.2}
Any $S_k$-invariant function $\psi$, can be realized through composition of an $S_k^{(0)}$-invariant function $\phi$ and a linear transformation $M$, i.e., $\psi = \phi \cdot M$. In addition, $\psi$ can be realised through the following form,
\begin{equation}
    \psi(x) = \rho \left(\sum_{i=1}^k \gamma \left(m_i^T x \right), \;  m_{k+1}^T x, \dots, m_n^T x\right),
\end{equation}
where $m_i$ is the $i^{th}$ row of $M$. 
\label{Conjugacy lemma}
\end{lemma}

\begin{proof}
Note that any $S_k$ is conjugate to $S_k^{(0)}$. Thus, $\exists$ $g \in S_n$ such that 
\begin{equation}
    S_k^{(0)} = g S_k g^{-1}
\end{equation}

Let $\psi: X \rightarrow R$ be an $S_k$-invariant function, i.e.,
\begin{align}
    \psi(x) &= \psi(h \cdot x), \quad \forall h \in S_k,  x \in X \nonumber  \\
    \psi(\left(g^{-1}g \right) \cdot x) &= \psi (\left(g^{-1}ug\right) \cdot x), \quad \forall  u \in S_k^{(0)}  \nonumber   \\
    (\psi  g^{-1})(g \cdot x) &= (\psi  g^{-1})  (u \cdot (g \cdot  x))  \nonumber   \\
    (\psi  g^{-1})  \left(Mx \right) &= (\psi  g^{-1})  (u \cdot \left(Mx \right))
    \label{S_k_0 invariant}
\end{align}
%From eq. (\ref{S_k_0 invariant}), we see that $\phi = \psi \cdot g^{-1}$ and $M=g$  are the desired $S_k^{(0)}$- invariant function and the linear transformation and $\phi = \psi \cdot M$. Applying Corollary (\ref{Corr Sk}) to $\phi$, we get the second part of the result.
From eq. (\ref{S_k_0 invariant}), we see that $\phi = \psi \cdot g^{-1}$ and $M=g$  are the desired $S_k^{(0)}$- invariant function and the linear transformation respectively and $\phi = \psi \cdot M$. We get the second part of the result by applying Corollary \ref{Corr Sk} to $\phi$.
\end{proof}

We could also relax the conjugacy condition, i.e., discover subgroups of type $S_k$ when $k$ itself is unknown. This is formalized in the following result.
\begin{theorem}[Subgroups of type $S_k$]
\label{Theorem 4.3}
Any $S_k$-invariant function ($k \leq n$) $\psi$,  can be realised using an $S_n$-invariant function and a linear transformation, in specific, it can be realised through the following form,
\begin{equation}
    \psi(x) = \left(  \phi \cdot \hat{M} \right) (x) =\rho \left( \begin{bmatrix}
(I - M)x \\
\sum_{i=1}^n \gamma \left(m_i^T x \right)
\end{bmatrix} \right)
\end{equation}
where $\hat{M} = \begin{bmatrix} I - M \\ M   \end{bmatrix}$ and \\
$\phi(y) = \Big[y_1, \dots, y_n, \; \sum_{i=1}^n \gamma(y_{n+i}) \Big]^T$
\label{permutation groups}
\end{theorem}

\begin{proof}
Since $S_k$ is conjugate to $S_k^0$, it is enough to prove the result for $S_k^0$-invariant function. Hence, the goal is to show that %$\begin{bmatrix}
%(I - M)X \\
%\sum_{i=1}^n \gamma \left(m_i^T X \right)
%\end{bmatrix}$ 
$(I-M)X \times \sum_{i=1}^n \gamma \left(m_i^T X \right)$is homeomorphic to $\sum_{m=1}^k \gamma(X_m) \times [0,1]^{n-k}$ (from Corollary \ref{Corr Sk} and Lemma \ref{Conjugacy lemma}) for some linear transformation $M$. Suppose,
\begin{equation}
M = \begin{bmatrix}
\label{M_matrix_S_k}
I_{k \times k} &0 \\
0 &0
\end{bmatrix},
\end{equation}

then,
\begin{equation}
\begin{bmatrix}
(I - M)x \\
\sum_{i=1}^n \gamma \left(m_i^T x \right)
\end{bmatrix}
=
\begin{bmatrix}
\textbf{0}^{(k)} \\
x_{k+1} \\
. \\
. \\
x_n \\
B + \sum_{i=1}^k \gamma \left(x_i \right)
\end{bmatrix},
\label{domain change S_k}
\end{equation}
where $B = (n-k)\gamma(0)$ and $\textbf{0}^{(k)}$ is $k$-dimensional zero vector.
%Since $(n-k) \gamma(0) + \sum_{i=1}^k \gamma \left(X_i \right)$ is homeomorphic to $\sum_{i=1}^k \gamma \left(X_i \right)$, the above claim follows.
Thus, from RHS of the eq. (\ref{domain change S_k}),  the above claim follows. (Note that, the function $Mx \mapsto \sum_{i=1}^n \gamma \left(m_i^T x \right)$ is $S_n$-invariant  and $\phi$ is $S_{2n}^n$-invariant function).
\end{proof}

%Until now, we considered the task of learning $S_k$-invariant functions using $S_n$-invariant functions and a %linear transformation. We now extend our method other subgroups of $S_n$. It turns out that it is indeed the case %as shown in the following result. 
We now extend our method to cyclic and dihedral subgroups of $S_n$ and state the following result.
% $\rho_{Z_k}$ (or $\rho_{D_{2k}}$) $\rho_{Z_k}$ (or $\rho_{D_{2k}}$)
\begin{theorem}[Cyclic and Dihedral subgroups]
\label{Theorem 4.4}
If $k|n$, any $\mathbb{Z}_k$-invariant (or $D_{2k}$-invariant) function $\psi$,  can be realised using a $\mathbb{Z}_n$-invariant (or $D_{2n}$-invariant) function $\phi$ and a linear transformation, in specific, it can be realised through the following form,
\begin{equation}
    \psi(x) = \left(  \phi \cdot \hat{M} \right) (x) 
\end{equation}
where $\hat{M} = \begin{bmatrix} M \\ I - L   \end{bmatrix}$ for some $M, L \in \mathbb{R}^{n \times n}$. 
\label{Cyclic groups and Dihedral groups}
\end{theorem}

\begin{proof}
In this proof, without loss of generality, we prove the result for $\mathbb{Z}_k^{(0)}$-invariant function. Suppose,

\begin{equation}
M = \begin{bmatrix}
I_{k \times k} &0 \\
I_{k \times k} &0 \\
\vdots &\vdots \\
I_{k \times k} &0 
\end{bmatrix}, \quad
L = \begin{bmatrix}
I_{k \times k} &0 \\
0 &0
\end{bmatrix},
\label{z-d-equation}
\end{equation}
Since $k | n$, we can stack the $I_{k \times k}$ matrices as shown in eq. (\ref {z-d-equation}). Then, $M : X \rightarrow X$ is defined as,
%_{Z_k \text{ action}}
\begin{align}
    x = [x_1, x_2 \dots x_n]^T \longmapsto Mx = &[x_1, x_2, \dots x_k, \nonumber \\ 
     &x_1, x_2, \dots, x_k,  \nonumber \\ 
     &\vdots   \nonumber \\ 
     &x_1, x_2, \dots, x_k]^T
\end{align}

Under the action of $\mathbb{Z}_k$ ($h \cdot x$, for some $h \in \mathbb{Z}_k$), we get that,
\begin{equation}
    x \xmapsto{h} x' = [x_u, x_{u+1}, \dots , x_k, x_1, \dots, x_{u-1}]^T
\end{equation}
which corresponds to ($g \cdot \left(Mx \right)$, for some $g \in \mathbb{Z}_n$),
\begin{align}
    Mx \xmapsto{g} &Mx' = [x_u, x_{u+1}, \dots , x_k, x_1, \dots, x_{u-1} \nonumber \\
    & x_u, x_{u+1}, \dots , x_k, x_1, \dots, x_{u-1} \nonumber \\
    &\vdots \nonumber \\
    & x_u, x_{u+1}, \dots , x_k, x_1, \dots, x_{u-1}]^T
\end{align}

Similarly, the converse is also true, i.e., $\mathbb{Z}_n$-action on $Mx$ corresponds to $\mathbb{Z}_k$-action on $x$. Hence, the $\mathbb{Z}_k$-invariant function of $x$ corresponds to $\mathbb{Z}_n$-invariance of $Mx$. Note that, the $\mathbb{Z}_n$-invariance of the function $\phi$ is with respect to the first $n$ elements (out of $2n$) of its input vector. Similar proof holds for dihedral groups ($D_{2k}$ and $D_{2n}$).
\end{proof}

%We conjecture that we could extend the above result to any generic class of subgroups. In this regard, we state our \textit{main theorem}.
The above set of techniques can also be extended to other classes of subgroups. In this regard, we state the following general result. 
\begin{theorem}
\label{Theorem 4.5}
Any $H$-invariant function $\psi$ can be learnt through composing a $G$-invariant function $\phi$ with a linear transformation $M$, i.e., $\psi = \phi \cdot M$ if  the following conditions hold, 
\begin{enumerate}
\item For any $h \in H,$ $\exists g \in G$ such that  $M(h \cdot x) = g \cdot \left(Mx\right), \: \forall x \in X$
\label{Cond 1}
\item For any $g \in G$ such  that $g \cdot \left(Mx\right) \in R(M)$,  $\exists h \in H$ such that  $M(h \cdot x) = g \cdot \left(Mx\right), \: \forall x \in X$, where $R(M)$ is the range of $M$. 
\label{Cond 2}
\end{enumerate}
%for every $h \in G_1$, $\exists$ some $g \in G_2$, such that 
\label{generalization}
\end{theorem}

\begin{proof} 
The claim directly follows from the following observations.

Condition (\ref{Cond 1}) states that, any action  $h \cdot x$ (action of $H$ on $X$) corresponds to an action $g \cdot (Mx)$ (action of $G$ on $R(M)$). 

Similarly, condition (\ref{Cond 2}) states that, any action  $g \cdot (Mx)$  corresponds to an action $h \cdot x$. 
\end{proof}

\section{DISCUSSION}
\label{discussion}
The underlying theme from the results stated in the previous section is that we could discover any subgroup belonging to a particular class of subgroups by learning a $G$-invariant function and a linear transformation. Depending on the class, the chosen G varies. We further elaborate on these observations in the following subsections.
\begin{figure}[htp]
\centering
\includegraphics[width=\columnwidth]{"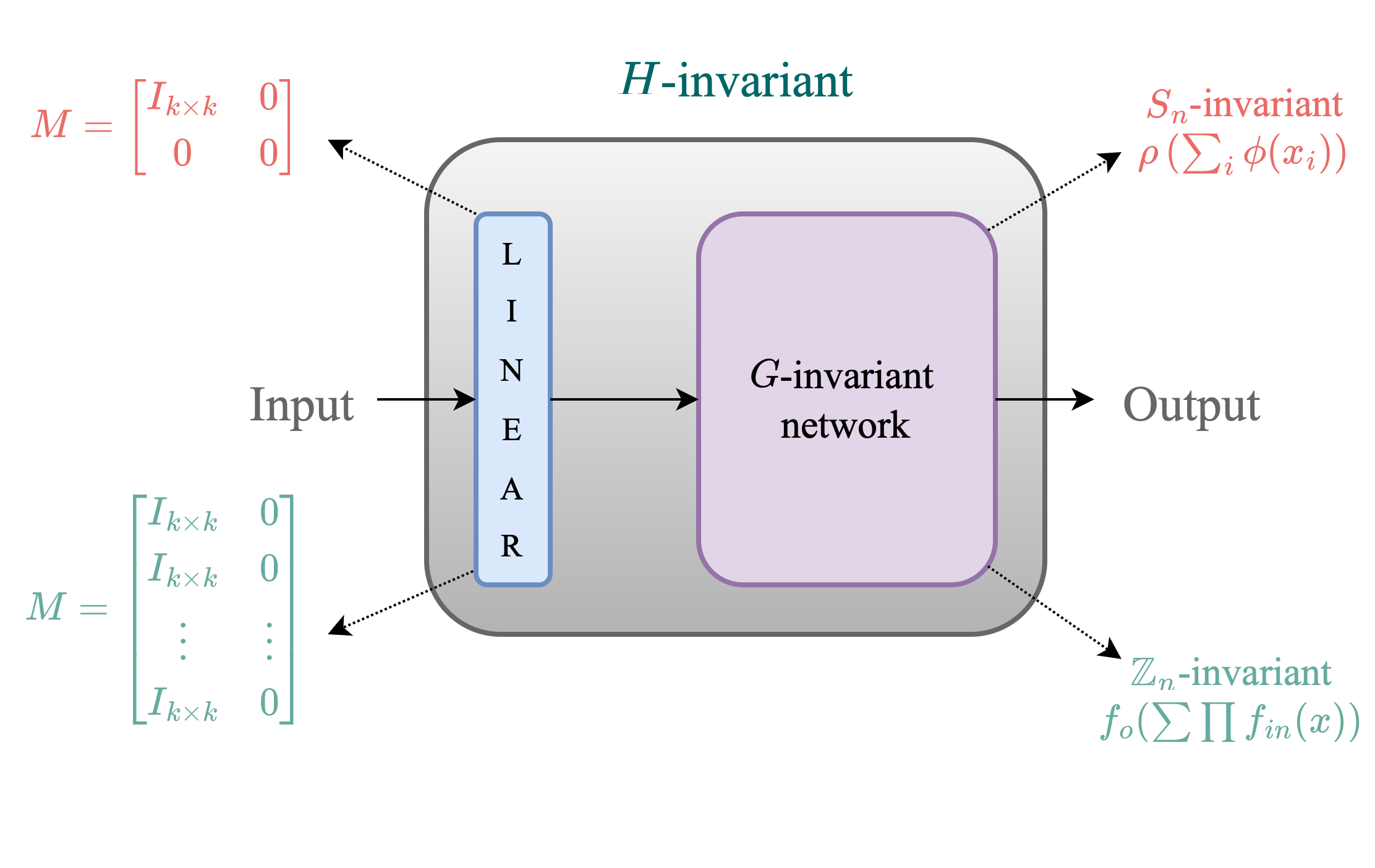"}
\caption{Generic framework for learning $H$-invariant function. The dotted arrows point towards specific examples of linear and $G$-invariant functions. The corresponding $H$-invariant functions are {\color{salmon}$S_k$-invariant} and {\color{teal}$\mathbb{Z}_k$-invariant}.}
\label{Generic method}
\end{figure}

\subsection{Conjugate Groups}
In Lemma \ref{Conjugacy lemma}, the class of subgroups corresponds to those of type $S_k$ (fixed $k$) and the corresponding $G$ can be $S_k^{0}$. We observe that, for a fixed $k$, even if we don't know the exact underlying subgroup $S_k$(a total of $\binom{n}{k}$ possibilities), we could learn this unknown subgroup. In addition, we also incorporate the canonical form of permutation invariant functions in the resulting architecture. Moreover, this result can be generalized to any class of conjugate subgroups, and the corresponding $G$ is one of these conjugate groups. The significance of this result lies in the fact that a variety of subgroups are related through conjugation. For instance, all $\mathbb{Z}_k$ form one conjugacy class for a given $k$, and so does $A_k$'s. 

This result is not entirely helpful if the underlying subgroup is normal since it is not conjugate to any other subgroup. However, this is not much of a hindrance since the only non-trivial proper normal subgroup of $S_n$ is $A_n, \; \forall n \geq 5$.

\subsection{$S_k$, $\mathbb{Z}_k$ and $D_{2k}$ Subgroups}
\label{Section 5.2}
Theorem \ref{permutation groups} focuses on subgroups of type $S_k$ (varying $k$ and $k \in \{1, 2, \dots, n \}$), and the corresponding $G$ is $S_n$ itself. We incorporate the canonical form of permutation invariant functions here as well. We observe that the number of such subgroups is $2^n - 1$ for a given $n$. Hence, we could learn any of these subgroups with the standard architecture of an $S_n$-invariant function and a linear transformation. Note that if $k$ is fixed, either of the architectural forms given by Lemma \ref{Conjugacy lemma} and Theorem \ref{permutation groups} is applicable. We will discuss the corresponding empirical results in the coming sections. 
Theorem \ref{Cyclic groups and Dihedral groups} considers subgroups of the cyclic $\mathbb{Z}_k$ and dihedral group $D_{2k}$. The corresponding $G$-invariant functions are of $\mathbb{Z}_n$ and $D_{2n}$, respectively.

\subsection{Generalization}
Theorem \ref{generalization} presents a general set of conditions to be satisfied to learn any $H$-invariant function using a $G$ invariant function and a linear transformation. As such, the previous results are specific cases of this Theorem. However, they provide explicit structures of the linear transformation $M$. These can help design appropriate training techniques to learn the optimum $M$, while the general result of Theorem \ref{generalization} can guide us towards discovering results for new classes of subgroups.

\subsection{Limitations}
The proposed framework presumes the knowledge of the underlying class of subgroups apriori (but not the exact subgroup) and an appropriate value of $n$ for $S_n$, $\mathbb{Z}_n$ or $D_{2n}$ invariant functions. The drawbacks mentioned here are interesting research directions to pursue in the future.

\section{EXPERIMENTS}
We evaluate the accuracy of our proposed method on image-digit sum and symmetric polynomial regression tasks. The problem of image-digit sum can be modified and cast as learning an $S_k$-invariant function, while the polynomial regression task intrinsically corresponds to learning a $G$-invariant function. These are summarized in the following subsections. 
\subsection{Image-Digit Sum}
This task aims to find the sum of $k$ digits using the MNISTm (\cite{loosli2007training}) handwritten digits dataset. It  consists of $8$ million gray scale $28 \times 28$ images of digits $\{0,1,...,9\}$. We employ a training set of $150k$ samples and a test set of $30k$ samples. We consider the following approaches for evaluation.
\begin{enumerate}
    \item \textbf{Deep Sets-$S_{k}$}:- $S_k$-invariant neural network proposed by \cite{zaheer2017deep}.
    \item \textbf{LSTM}:- LSTM network as mentioned in \cite{zaheer2017deep}.
    \item \textbf{Proposed method}:- A linear layer followed by an $S_n$-invariant network.
\end{enumerate}
For the LSTM network and the proposed method, the input is a random sample of  $n$ ($n$ = 10) images, and the target is the sum of $k$ ($k$ less than $n$) digit labels. We run separate experiments for each of $k \in \{1,3,5,7,9\}$. Since all $n$ images are given as input, the two approaches are agnostic of the underlying subgroup. However, we feed only these $k$ of these images as input for the first approach, while the target output remains the same. As such, this task is equivalent to learning an $S_k$-invariant function.

\subsection{Symmetric Polynomial Regression}
We evaluate the performance of our method on symmetric polynomial regression tasks as discussed in \cite{kicki2020computationally}, primarily for subgroups of $\mathbb{Z}_{10}$ and $\mathbb{Z}_{16}$. For all our experiments, we utilize a $\mathbb{Z}_{n}$-invariant neural network with a Sum-Product layer as discussed in \cite{kicki2020computationally} and a linear layer. First, we run our experiments for subgroups of $\mathbb{Z}_{10}$, i.e., $\mathbb{Z}_5$ and the group itself (trivial subgroup). We then access the performance for subgroups of $\mathbb{Z}_{16}$, namely $\mathbb{Z}_2$, $\mathbb{Z}_4$, $\mathbb{Z}_8$, $\mathbb{Z}_{16}$ using a similar architectural design. We consider the following approaches for evaluation.

\begin{enumerate}
    \item \textbf{G-invariant}:- $\mathbb{Z}_k$-invariant neural network proposed by \cite{kicki2020computationally}. In this context, $G=\mathbb{Z}_k$. 
    \item \textbf{Simple-FC}:- A stack of fully-connected feedforward layers.
    \item \textbf{Conv-1D}:- A simple convolutional neural network and feedforward layers.
    \item \textbf{Proposed method}:- A linear layer followed by a $\mathbb{Z}_n$-invariant network.
\end{enumerate}
The architectural details of the models considered in our experiments are discussed in the appendix section.

\section{RESULTS}
\subsection{Image-Digit Sum}
\label{Result section for Image-Digit Sum}
The test mean absolute errors (MAEs) for the image-digit sum task are shown in Table \ref{MAE for Image Digit-Sum task}. We observe that the proposed method outperforms the LSTM baseline and is competitive with respect to the Deep Sets method (k input images) when the underlying subgroup $S_{k}$ is known. In addition, our method converges faster when compared to the LSTM network, which is apparent from the plots for the training and validation errors in Figure \ref{Training and Validation loss (MAE) for Image-Digit Sum using MNIST dataset.}.
\begin{table}[htp]
\caption{MAE [$\times 10^{-2}$] for Image Digit-Sum task} 
\label{MAE for Image Digit-Sum task}
\begin{center}
\resizebox{\columnwidth}{!}{\begin{tabular}{||l| l| l| l| l| l||}
\hline 
\textbf{Method}  &\textbf{$S_1$} &\textbf{$S_3$} &\textbf{$S_5$} &\textbf{$S_7$} &\textbf{$S_9$} \\
\hline
Deep Sets-$S_{k}$   &  $5.61 \pm 0.35$ & $7.66 \pm 0.26$ & $8.02 \pm 0.2$ & $7.68 \pm 0.43$ &$6.97 \pm 0.39$ \\
Proposed     &   $5.73 \pm 0.39$  & $7.78 \pm 0.49$ & $8.19 \pm 0.36$ & $7.84 \pm 0.41$ &$7.26 \pm 0.58$ \\
LSTM         &   $6.23 \pm 0.53$  & $9.65 \pm 0.57$ & $11.98 \pm 0.46$ & $13.35 \pm  1.02$ &$12.92 \pm 1.42$ \\
\hline
\end{tabular}}
\end{center}
\end{table}

\begin{figure}[h]
\includegraphics[width=\columnwidth, trim={3cm 0 2.5cm 0}, clip]{"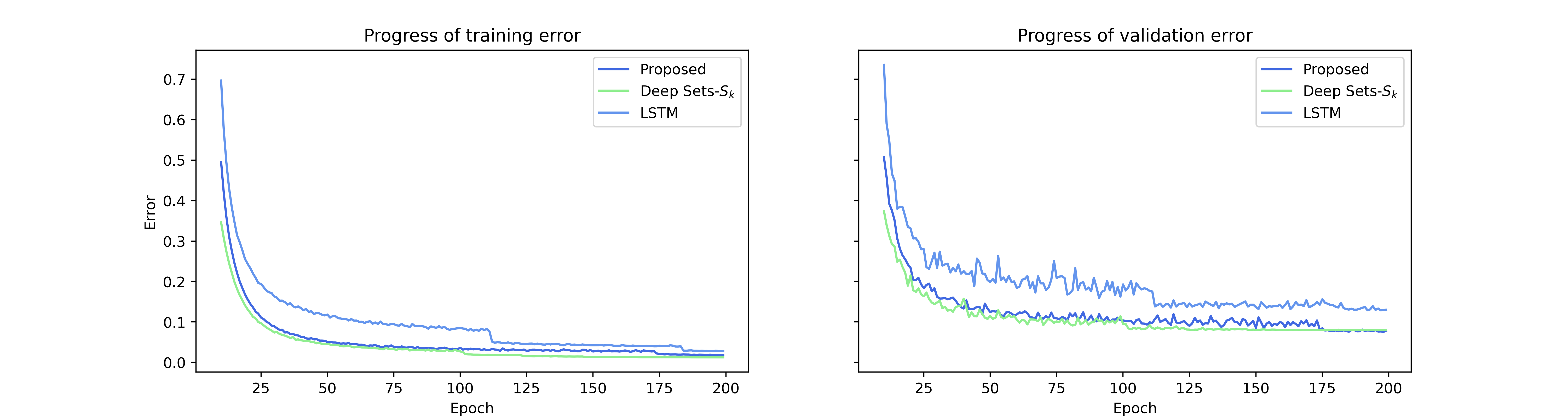"}
\caption{Training and Validation loss (MAE) for Image-Digit Sum using MNIST dataset.}
\label{Training and Validation loss (MAE) for Image-Digit Sum using MNIST dataset.}
\end{figure}

\subsection{Symmetric Polynomial Regression}
\label{Result section for Z-invariant Symmetric Polynomial Regression}
In the $\mathbb{Z}_{k}$-invariant polynomial regression task, we train our models for 2500 epochs for each of the subgroups of $\mathbb{Z}_{5}$ and $\mathbb{Z}_{10}$.
%we refer the reader to the appendix section for a detailed analysis of our results.
%we refer the reader to the appendix section for a detailed analysis of our results. 

In Table \ref{Z_5:Z_10}, \ref{Z_10:Z_10} and \ref{Z_4:Z_16} we compare the given baselines with our proposed method for the task of discovering unknown subgroups. Our method outperforms the Simple-FC and Conv-1D baseline networks for each of the given subgroups. As expected, it does not match the baseline architecture, the $\mathbb{Z}_k$-invariant network (the subgroup is known apriori for this baseline) by a significant margin for each of the diverse set of subgroups we have considered in this task. However, in a few cases, we observe large standard deviations and attribute such values to outliers. A detailed version of our results and the mathematical definition of the polynomials is presented in the appendix section.

From Figure \ref{Z5THROUGHZ10}, it is evident that the $\mathbb{Z}_{5}$-invariant function outperforms both our method and the baselines by a significant margin. The Simple FC and Conv-1D networks have very similar performances and show no prominent effect, even with an increase in data size.
\begin{table}[htp]
%\vspace{-1.3em}
\caption{MAE $[\times10^{-2}]$ for $\mathbb{Z}_{5} : \mathbb{Z}_{10}$} 
\label{Z_5:Z_10}
\begin{center}
\resizebox{\columnwidth}{!}{\begin{tabular}{||l| l| l| l||}
\hline 
\textbf{Method}  &\textbf{Train} &\textbf{Validation} &\textbf{Test} \\
\hline
$\mathbb{Z}_{5}$-invariant          &$2.65 \pm 0.91$       &$7.32 \pm 0.55$	   &$7.53 \pm 0.576$ \\
Proposed        &$4.48 \pm 1.25$	       &$24.56 \pm 6.93$    &$24.78  \pm 6.45$ \\
Conv-1D         &$20.90 \pm 4.91$	   &$32.96 \pm 1.31$	   &$32.33 \pm 1.18$\\
Simple-FC       &$23.86 \pm 3.87$	   &$33.57 \pm 2.07$	   &$33.14 \pm 2.11$\\
\hline
\end{tabular}}
\end{center}
\end{table}
%\vspace{-3.3em}

\begin{table}[htp]
\caption{MAE $[\times10^{-2}]$ for $\mathbb{Z}_{10} : \mathbb{Z}_{10}$} 
\label{Z_10:Z_10}
\begin{center}
\resizebox{\columnwidth}{!}{\begin{tabular}{||l| l| l| l||}
\hline 
\textbf{Method}  &\textbf{Train} &\textbf{Validation} &\textbf{Test} \\
\hline
$\mathbb{Z}_{10}$-invariant           &$6.89 \pm 1.31$	      &$16.68 \pm 0.55$	&$17.16 \pm 0.56$ \\
Proposed        &$14.52 \pm 1.72$	  &$39.69 \pm 4.13$	&$40.11 \pm 4.17$ \\
Conv-1D        &$35.71 \pm 2.71$	   &$52.96 \pm 0.70$	   &$50.63 \pm 1.33$\\
Simple-FC       &$46.13 \pm 2.27$	      &$54.62 \pm 1.34$	&$51.64 \pm 0.89$\\
\hline
\end{tabular}}
\end{center}
\end{table}
%\vspace{-3.3em}

\begin{table}[htp]
\caption{MAE $[\times10^{-2}]$ for  $\mathbb{Z}_{4} : \mathbb{Z}_{16}$} 
\label{Z_4:Z_16}
\begin{center}
\resizebox{\columnwidth}{!}{\begin{tabular}{||l| l| l| l||}
\hline 
\textbf{Method}  &\textbf{Train} &\textbf{Validation} &\textbf{Test} \\
\hline
$\mathbb{Z}_{4}$-invariant         &$1.21 \pm 0.25$	      &$3.41\pm0.4$	&$3.54\pm0.39$ \\
Proposed        &$3.32\pm1.65$	  &$23.70\pm4.87$	&$24.69\pm5.25$ \\
Conv-1D        &$8.39 \pm 3.02$	   &$31.34 \pm 0.77$	   &$31.10 \pm 0.87$\\
Simple-FC       &$7.27\pm5.03$	      &$30.82\pm1.74$	&$30.83\pm1.61$\\
\hline
\end{tabular}}
\end{center}
%\vspace{-1.3em}
\end{table}

\subsection{Effect of the data size on the performance}
This section aims to assess the effect of the dataset size in learning $\mathbb{Z}_k$-invariant functions using our proposed method and hope to gain a better understanding in such a setting. To analyze our model performance with respect to data size, we use 16, 32, and 64 data points for training (as mentioned in \cite{kicki2020computationally}, we randomly sample these values from [0,1]) and use 480 and 4800 as validation and test sets respectively to assess the generalization ability for each of these methods as mentioned above. We report the mean and standard deviation values across 10 randomly initialized iterations.
%The goal of this section is to assess the effect of the dataset size in learning $Z$-invariant functions using our proposed method and hope to gain a better understanding in such a setting. To analyze our model performance with respect to data size, we use 16, 32, and 64 data points for training (as mentioned in \cite{kicki2020computationally}, we randomly sample these values from [0,1]) and use 480 and 4800 as validation and test sets respectively to assess the generalization ability for each of these methods as mentioned above. We report the mean and standard deviation values across ten randomly initialized iterations.
%\begin{figure}[h]
%\centering
%\includegraphics[width=\columnwidth]{"Plot_1_AISTATS_DeepSets.png"}
%\includegraphics[width=0.7\columnwidth,height=0.5\columnwidth]{"img16.png"}
%\caption{The MAE value comparisons using the test dataset for all the models we have considered in this task for learning $\mathbb{Z}_5$ invariant network. The x-axis represents the size of the training set $(16, 32, 64)$.}
%\label{Z5THROUGHZ10}
%\vspace{-1.5em}
%\end{figure}

\begin{figure}[htp!]
\centering
\includegraphics[width=0.7\columnwidth,height=0.5\columnwidth]{"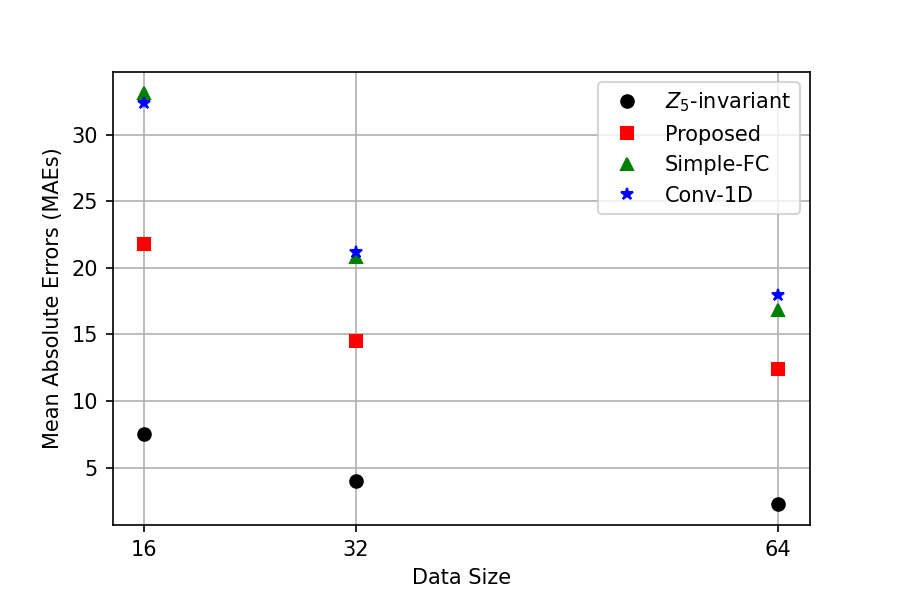"}
\caption{The MAE value comparisons using the test dataset for all the models we have considered for the $\mathbb{Z}_{5} : \mathbb{Z}_{10}$ task. The $X$-axis represents the size of the training set $(16, 32, 64)$.}
\label{Z5THROUGHZ10}
%\vspace{-1.5em}
\end{figure}
We also examine the Simple-FC and Conv-1D network by increasing its parameter count, i.e., varying the number of neurons in each layer. However, we observe no significant gains in doing so, as mentioned in the appendix section for at least a few subgroups.
\subsection{Interpretability}
%\begin{figure}[h]
%\vspace{-1em}
%\centering
%\begin{subfigure}
%  \centering
%  \includegraphics[width=.4\linewidth]{M_matrix_Idx_3.png}
%  \label{M_matrix_Idx_3}
%\end{subfigure}
%\begin{subfigure}
%  \centering
%  \includegraphics[width=.4\linewidth]{M_matrix_Idx_5.png}
%  \label{M_matrix_Idx_5}
%\end{subfigure}
%\caption{$M$ matrix for $S_5$ and $S_9$ after training. }
%\label{M_matrix_Idx}
%\vspace{-0.7em}
%\end{figure}
\begin{figure}[htp!]
    \centering
    \subfloat[\centering $S_5$]{{\includegraphics[width=.5\linewidth, trim={0cm 0 0cm 3cm}, clip]{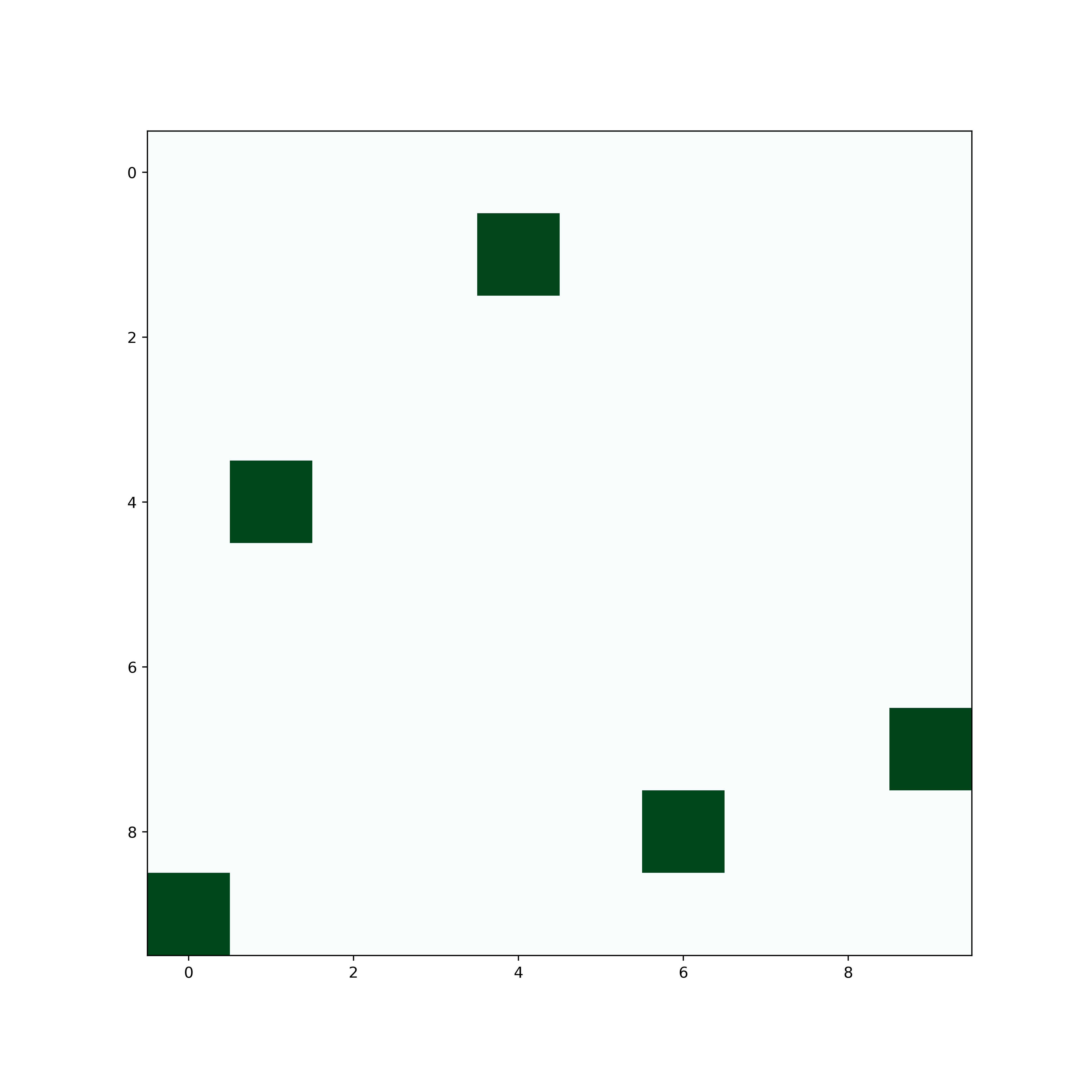} }}%
    %\qquad
    \subfloat[\centering $S_9$]{{\includegraphics[width=.5\linewidth, trim={0cm 0 0cm 3cm}, clip]{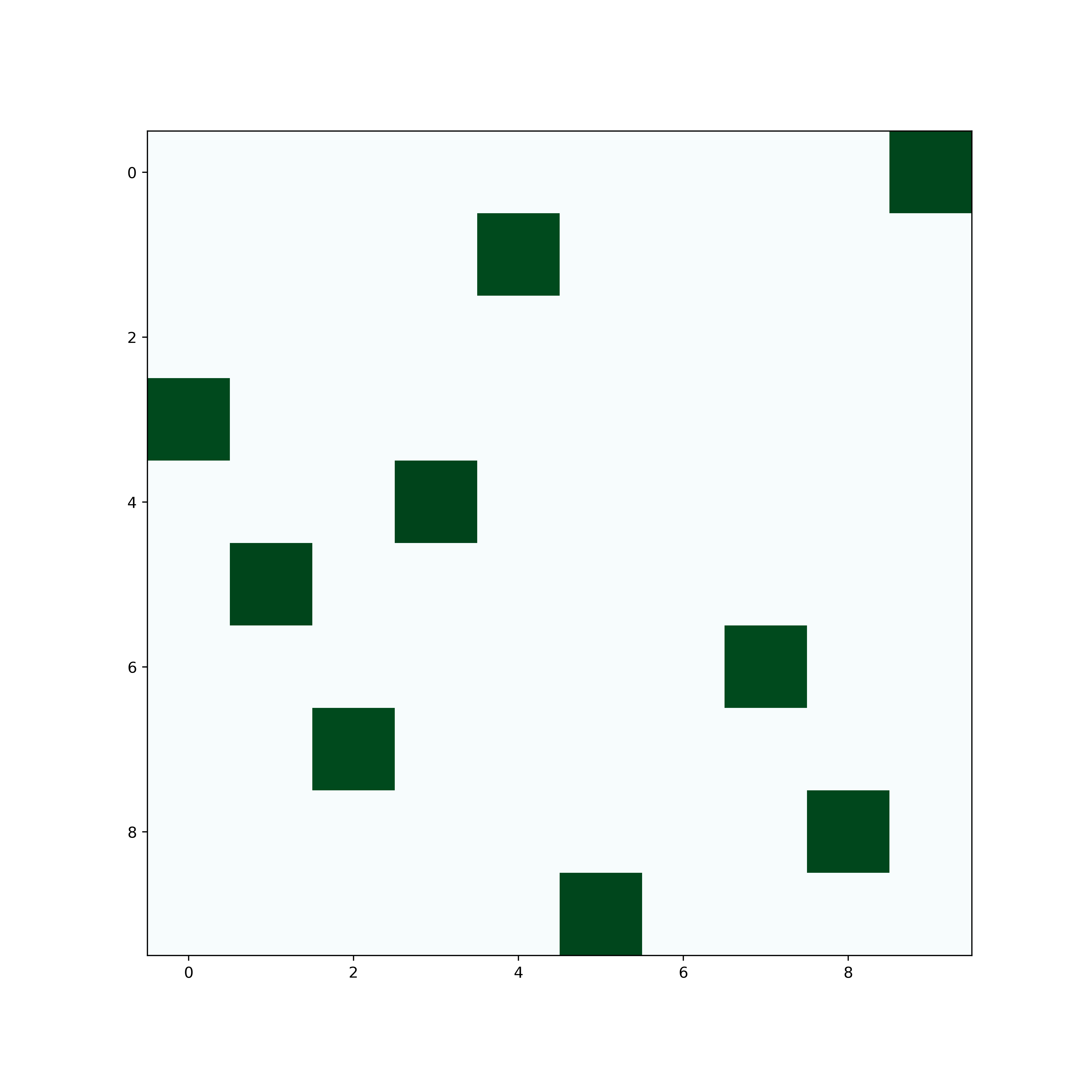} }}%
    \caption{$M$ matrices for $S_5$ and $S_9$ after training. }
    \label{M_matrix_Idx}
\end{figure}

\subsubsection{Image-Digit Sum}
The resulting M matrix is interpretable, and we consistently observe the expected pattern for the image-digit sum task. Note that any row-permuted version of the matrix structure, as shown  in eq. {\color{red} (\ref{M_matrix_S_k})} will work since the transformed space is still homeomorphic. The $M$ matrices for $S_5$ and $S_9$ (extracted after training) are depicted in Figure \ref{M_matrix_Idx}. The columns with dark green squares match the actual indices.

% \begin{figure}[h]
% %\vspace{-1.2em}
% \centering
% \includegraphics[width=\columnwidth, trim={2.5cm 0 2cm 0}, clip]{"small_size_both_matrices.png"}
% \caption{(a) $M$ matrix for $\mathbb{Z}_4: \mathbb{Z}_{16}$ (b) Reference matrix.}
% \label{small_size_both_matrices}
% %\vspace{-1em}
% \end{figure}

\begin{figure}[h]
%\vspace{-1.2em}
\centering
\includegraphics[width=\columnwidth]{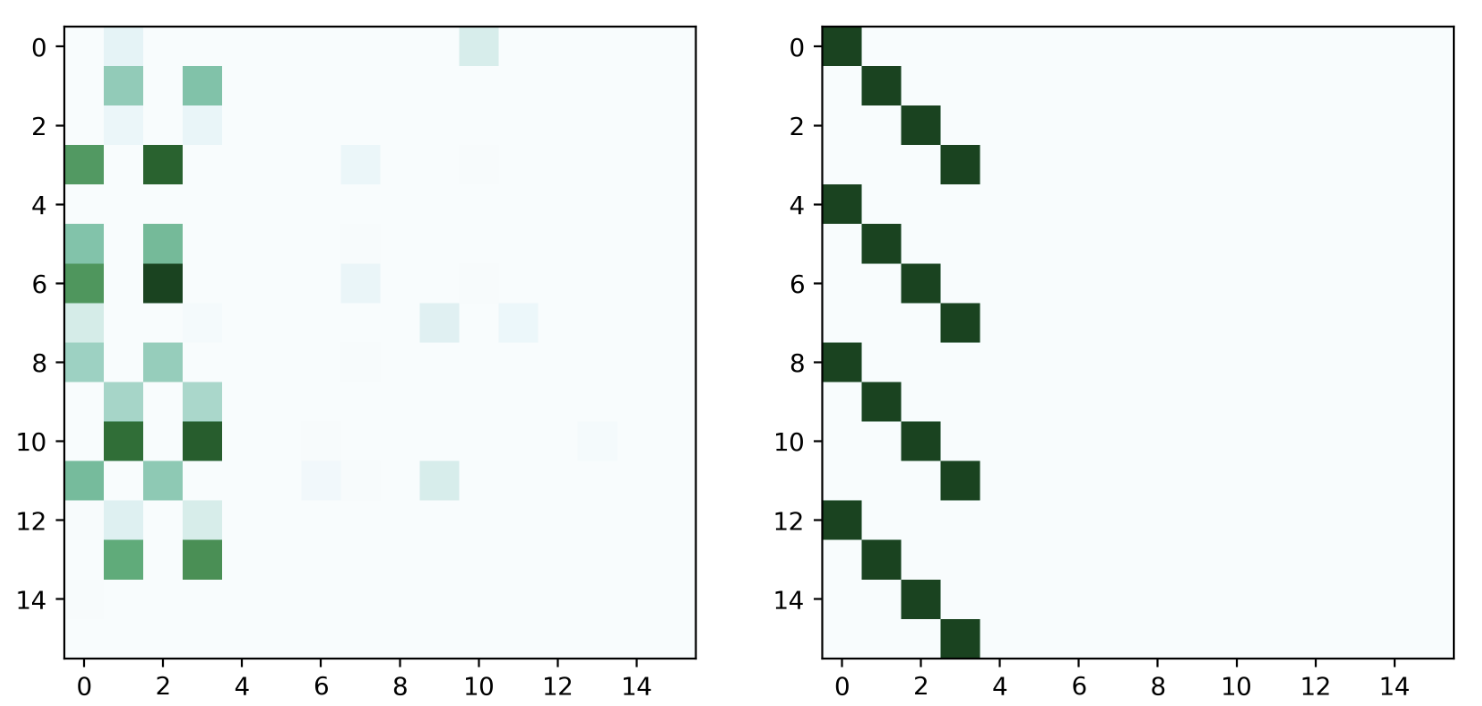}
\caption{(a) $M$ matrix for $\mathbb{Z}_4: \mathbb{Z}_{16}$ (b) Reference matrix.}
\label{small_size_both_matrices}
%\vspace{-1em}
\end{figure}

\subsubsection{Polynomial Regression}
%2. \textbf{Polynomial Regression}: 
We observe that the $M$-matrix extracted after training (Figure ({\color{red}}\ref{small_size_both_matrices}{\color{red}.a})) does not exactly capture the expected pattern, i.e., a stack of identity matrices (Figure ({\color{red}}\ref{small_size_both_matrices}{\color{red}.b})), even though it nearly masks most of the irrelevant columns ($n-k$). The former behavior (lack of exact structure) explains the difference in performance with respect to the $\mathbb{Z}_k$-invariant network, while the latter (masking behavior) describes the superior model performance compared to other baselines. Also, the masking of irrelevant columns already conveys the underlying subgroup; thus, we use this information to estimate the true indices. We estimate the significant indices using the $L1$-norm of columns of $M$ and the mean as the threshold. The results (for different number of training data points $N$ and different $\mathbb{Z}_k:\mathbb{Z}_n$'s) of the success rate of the estimation are given in Table {\color{red}} \ref{Estimation Accuracy}, where we count the estimation as success when the estimated indices exactly match the true indices; otherwise, as a failure. We run each experiment for $10$ trials. We get high estimation accuracy in most of the cases except for $N = 16$. The estimated indices can be used to run a $\mathbb{Z}_k$-invariant network (or proposed method with fixed $M$) and obtain better performance on regression tasks.

\begin{table}[htp]
%\vspace{-1em}
\caption{Estimation Accuracy (in \%)} 
\label{Estimation Accuracy}
\begin{center}
%\resizebox{0.5\columnwidth}{!}{
{\begin{tabular}{||l| l| l| l||}
\hline 
\textbf{${Z}_k:{Z}_n$}  &\textbf{$16$} &\textbf{$32$} &\textbf{$64$} \\
\hline
${Z}_4:{Z}_{16}$   &  $100$ & $100$ & $100$ \\
${Z}_5:{Z}_{10}$   &   $80$  & $100$ & $100$ \\
${Z}_8:{Z}_{16}$   &   $30$  & $80$ & $100$ \\
\hline
\end{tabular}}
\end{center}
\end{table}

\section{CONCLUSION}
In this work, we studied the problem of discovering the underlying subgroup of $S_n$, i.e., learning a $H$-invariant function where $H$ is an unknown subgroup of $S_n$. We proved that we could learn any $H$-invariant function using a $G$-invariant function and a linear transformation provided $H$ belongs to a specific class of subgroups. We considered various subgroups, such as conjugate subgroups, permutation subgroups of $k$ elements, and cyclic and dihedral subgroups, and illustrated unique structures of the corresponding linear transformations. We demonstrated the validity of our theoretical analysis through empirical results. We also discussed the limitations of our method, which may lead to exciting research directions in the future.

%\subsection*{References}
%\newpage
%\newpage
\bibliography{main}
\onecolumn
\aistatstitle{Supplementary Materials}
\vspace{-2.3cm}
$$$$
$$$$
%\hfill \break
\vspace{-1.25cm}
%\vspace{0.1cm}
%$$$$
\section{Appendix}
%\vspace{1cm}
\subsection{Remarks regarding theoretical results}
\begin{itemize}
    \item To obtain the function $\rho$ mentioned in Lemma \ref{Lemma 4.2}, Theorem \ref{Theorem 4.3}, and Theorem \ref{Theorem 4.4} of the main paper, we use the same technique presented in the proof of Corollary \ref{Corr Sk}, i.e., $\rho = \psi E^{-1}$, where $E$ is the corresponding homeomorphism.
    \item Lemma \ref{Lemma 4.2}, Theorem \ref{Theorem 4.3}, and Theorem \ref{Theorem 4.4} are special cases of Theorem \ref{Theorem 4.5}. This claim directly follows once we specify the corresponding group actions. As such, the proofs of Lemma \ref{Lemma 4.2} and Theorem \ref{Theorem 4.4} already describe the required group actions. However, this is not obvious in Theorem \ref{Theorem 4.3}. In that case, we observe the following:
\begin{align}
    x = [x_1, x_2 \dots x_n]^T \longmapsto Mx = &[0,0, \;\overset{k}{\dots}\;, 0, \; x_{k+1}, x_{k+2}, \dots x_n,  \nonumber \\
    &x_1, x_2, \dots x_k, \; 0,0, \;\overset{n-k}{\dots}\;, 0]^T,
\end{align}
where $M = \begin{bmatrix}
I_{k \times k} &0 \\
0 &0
\end{bmatrix}$ as mentioned in eq. (\ref{M_matrix_S_k}) of the main paper.

Under the action of $S_k^{(0)}$ \Big($h \cdot x, \;$ for some $h \in S_k^{(0)}$ \Big), we get,
\begin{equation}
    x \xmapsto{h} x' = [x_{h(1)}, x_{h(2)}, \dots , x_{h(k)}, x_{k+1}, \dots, x_n]^T
\end{equation}
which corresponds to \Big($g \cdot \left(Mx \right)$, for some $g \in S^n_{2n}$ \Big) \Big($S^n_{2n}$ is the group of permutations of first $n$ elements out of $2n$ elements\Big),
\begin{align}
    Mx \xmapsto{g} Mx' = &[0, 0, \;\overset{k}{\dots}\; 0, \; x_{k+1}, x_{k+2}, \dots, x_n, \nonumber \\
    &x_{g(1)}, x_{g(2)}, \dots , x_{g(k)},\; 0, 0, \;\overset{n-k}{\dots}\;, 0]^T 
\end{align}    
\end{itemize}
Similary, the action of $S^n_{2n}$ on $R(M)$ (range of M) corresponds to the action of $S_k^{(0)}$ on $X$.

In the following subsections, we describe the architectures of various models and additional resutls considered in our experiments. 
\begin{table}[htp]
\vspace{-1em}
\caption{MAE [$\times 10^{-2}$] Image Digit-Sum task} 
\label{MAE for baselines}
\begin{center}
\resizebox{0.7\columnwidth}{!}{\begin{tabular}{||l| l| l| l| l| l||}
\hline 
\textbf{Method}  &\textbf{$S_1$} &\textbf{$S_3$} &\textbf{$S_5$} &\textbf{$S_7$} &\textbf{$S_9$} \\
\hline
LSTM         &   $6.23 \pm 0.53$  & $9.65 \pm 0.57$ & $11.98 \pm 0.46$ & $13.35 \pm  1.02$ &$12.92 \pm 1.42$ \\
Conv-1D   &  $36.32 \pm 0.12$ & $19.11 \pm 0.49$ & $27.92 \pm 0.41$ & $35.42 \pm 0.19$ &$40.83 \pm 0.11$ \\
Simple FC     &   $25.26 \pm 0.01$  & $18.18 \pm 0.15$ & $35.27 \pm 0.07$ & $44.51 \pm 0.58$ &$51.79 \pm 0.89$ \\
\hline
\end{tabular}}
\end{center}
\end{table}
\subsection{Image-Digit sum}
\begin{enumerate}
    \item \textbf{Deep Sets-$S_{k}$}:- $S_k$-invariant neural network proposed by \cite{zaheer2017deep}. It consists of two networks, $\gamma$, and $\rho$. Each element in the input is passed through the $\gamma$ network, followed by the sum operation. The result is then fed to the second network $\rho$. The network $\gamma$ is a feed-forward network consisting of three \textit{dense} layers with \textit{tanh} activation, and the second network is a \textit{dense} layer.    
    \item \textbf{LSTM}:- The LSTM network used for comparison in \cite{zaheer2017deep}. It consists of two \textit{dense} layers, an \textit{LSTM} layer followed by two \textit{dense} layers. The activation used is \textit{tanh} function.
     \item \textbf{Proposed method}:- An $S_n$-invariant network follows a linear layer. The $S_n$-invariant network has the same architecture as Deep Sets (the first approach) except the input layer. 
\end{enumerate}

\begin{table}[h]
\vspace{-0.75em}
\caption{Definitions of the various polynomials used in the main paper.} 
\label{Poly Def}
\begin{center}
\resizebox{0.5\columnwidth}{!}{\begin{tabular}{||l| l||}
\hline 
\textbf{INVARIANCE}  &\textbf{POLYNOMIAL}  \\
\hline
$\mathbb{Z}_2:\mathbb{Z}_{16}$    &$x_1x_2^2+x_2x_1^2$		  \\
&\\
$\mathbb{Z}_4:\mathbb{Z}_{16}$    &$x_1x_2^2+x_2x_3^2+x_3x_4^2+x_4x_1^2$		\\
&\\
$\mathbb{Z}_5:\mathbb{Z}_{10}$    &$x_1x_2^2+x_2x_3^2+x_3x_4^2+x_4x_5^2+x_5x_1^2$	      \\
&\\
$\mathbb{Z}_8:\mathbb{Z}_{16}$    &$x_1x_2^2+x_2x_3^2+ ... + x_7x_8^2+x_8x_1^2$		  \\
&\\
$\mathbb{Z}_{10}:\mathbb{Z}_{10}$    &$x_1x_2^2+x_2x_3^2+ ... + x_9x_{10}^2+x_{10}x_1^2$		  \\
&\\
$\mathbb{Z}_{16}:\mathbb{Z}_{16}$    &$x_{1}x_{2}^2+x_{2}x_{3}^2+ ... + x_{15}x_{16}^2+x_{16}x_{1}^2$		\\
\hline
\end{tabular}}
\end{center}
\end{table}
\vspace{-2em}

\subsection{Comparison between $S_k$-invariant networks with backbone as $S_k^{(0}$ and $S_n$}
As specified in Section \ref{Section 5.2} of the main paper, any $S_k$-invariant network can be realized through either an $S_k^{(0)}$ or an $S_n$-invariant network and a linear layer when $k$ is fixed. In general, we observed that the $S_n$-invariant network as the backbone does better than the $S_k^{(0)}$ network. We attribute this to the expressivity power of the linear transformation (based on its specific structure) when an $S_n$ invariant network is used.

%We have plotted the graphs of validation errors for $k = 4, 5 \& 6$ in Fig. (). The test errors are given in Table. (). We observed no significant difference between these approaches except for outliers with high errors for the first case.

\subsection{Symmetric Polynomial Regression}

\begin{enumerate}
    \item \textbf{G-invariant}:- $\mathbb{Z}_k$-invariant network implemented using the design described in \cite{kicki2020computationally}. As discussed in \cite{kicki2020computationally}, it is a composition of a $\mathbb{Z}_k$-equivariant network and a Sum-Product Layer.  It then uses a Multi-Layer Perceptron to process the $\mathbb{Z}_k$-invariant representation of the input and thus predicts the polynomial output. The network is thus invariant under the action of the given permutation subgroup $\mathbb{Z}_k$.    
    \item \textbf{Simple-FC}:-  This is an abbreviation of a fully-connected neural network without the Reynolds operator, i.e., group averaging for this baseline implementation (\cite{derksen2001computational}). 
    \item \textbf{Conv-1D}:- This is an abbreviation of the 1D Convolutional neural network equipped with fully-connected layers. 
    \item \textbf{Proposed method}:- To discover the underlying subgroup, we use a  $\mathbb{Z}_n$-invariant neural network with the addition of a linear layer. The architectural design of the $\mathbb{Z}_n$-invariant function is the same as the $G$-invariant network.
\end{enumerate}
The hyperparameters of the above models are given in Table. \textcolor{red}{(6-9)} of \cite{kicki2020computationally}.
\begin{table}[h]
\caption{MAE $[\times10^{-2}]$ Polynomial Regression} 
\label{Z_8:Z_16}
\begin{center}
\resizebox{0.5\columnwidth}{!}{\begin{tabular}{||l| l| l| l||}
\hline 
\textbf{Method}  &\textbf{$\mathbb{Z}_2: \mathbb{Z}_{16}$} &\textbf{$\mathbb{Z}_8: \mathbb{Z}_{16}$} &\textbf{$\mathbb{Z}_{16}:\mathbb{Z}_{16}$} \\
\hline
$\mathbb{Z}_{k}$-invariant    &$1.26\pm0.25$	      &$14.30\pm1.04$	&$17.16 \pm 0.59$ \\
Proposed &$22.27\pm4.25$	  &$39.31\pm5.12$	&$40.11 \pm 4.17$ \\
Conv-1D &$21.67\pm0.69$	  &$50.15\pm1.03$	&$68.38\pm3.13$ \\
Simple FC       &$21.61\pm1.54$	      &$45.92\pm2.83$	&$51.64 \pm 0.89$\\
\hline
\end{tabular}}
\end{center}
\end{table}
%\vspace{-2em}
 \subsubsection{Additional results}
%The train, validation, and test errors for the $\mathbb{Z}_{2} : \mathbb{Z}_{16}$, $\mathbb{Z}_{8} : \mathbb{Z}_{16}$ and $\mathbb{Z}_{16} : \mathbb{Z}_{16}$ invariant functions are provided in Table. \ref{Z_2:Z_16}, \ref{Z_8:Z_16} and \ref{Z_16:Z_16} respectively. The details of the all the polynomials used in our work are presented in Table. \ref{Poly Def}.
The test errors for $\mathbb{Z}_{2} : \mathbb{Z}_{16}$, $\mathbb{Z}_{8} : \mathbb{Z}_{16}$ and $\mathbb{Z}_{16} : \mathbb{Z}_{16}$ invariant functions are provided in Table \ref{Z_8:Z_16}. The details of all the polynomials used in our work are presented in Table \ref{Poly Def}.
%\newpage
%\bibliography{mybib}
\end{document}